\documentclass[letterpaper, 10 pt, conference]{ieeeconf}

\usepackage{float}
\usepackage[ruled,vlined,linesnumbered]{algorithm2e}
\usepackage{yfonts}
\usepackage{url}
\usepackage{multirow}
\usepackage[utf8]{inputenc}
\usepackage[T1]{fontenc}
\usepackage{ifthen}
\newboolean{PrintVersion}
\setboolean{PrintVersion}{false}
\usepackage{amsmath,amssymb,amstext}
\usepackage[pdftex]{graphicx} % For including graphics N.B. pdftex graphics driver 
\usepackage{booktabs}
\usepackage{siunitx}
\usepackage{multirow}
\usepackage{array}

\usepackage{color}
\usepackage{xcolor}

\usepackage[font=footnotesize]{caption}
\usepackage{subcaption}

\newtheorem{theorem}{Theorem}[section]

\newtheorem{lemma}[theorem]{Lemma}

\newtheorem{definition}[theorem]{Definition}
\newtheorem{problem}[theorem]{Problem}

\newtheorem{claim}[theorem]{Claim}
\newtheorem{remark}[theorem]{Remark}

\newtheorem{observation}[theorem]{Observation}

\title{\LARGE \textbf{Energy-Aware Route Planning for a Battery-Constrained Robot \\with Multiple Charging Depots}}
\author{Ahmad Bilal Asghar, Pratap Tokekar
}

\usepackage{lipsum}

\begin{document}
\maketitle
\begin{abstract} 
This paper considers energy-aware route planning for a battery-constrained robot operating in environments with multiple recharging depots. The robot has a battery discharge time $D$, and it should visit the recharging depots at most every $D$ time units to not run out of charge. The objective is to minimize robot's travel time while ensuring it visits all task locations in the environment. We present a $O(\log D)$ approximation algorithm for this problem. We also present heuristic improvements to the approximation algorithm and assess its performance on instances from TSPLIB, comparing it to an optimal solution obtained through Integer Linear Programming (ILP). The simulation results demonstrate that, despite a more than $20$-fold reduction in runtime, the proposed algorithm provides solutions that are, on average, within $31\%$ of the ILP solution.
\end{abstract}
\section{Introduction}
In recent years, battery-powered robots have seen significant development and deployment in a wide range of fields, including but not limited to autonomous drones and robots engaged in surveillance missions~\cite{basilico2012patrolling,asghar2016stochastic}, autonomous delivery systems~\cite{jennings2019study}, agriculture~\cite{wei2018coverage}, forest fire monitoring~\cite{merino2012unmanned}, and search and rescue operations~\cite{hayat2017multi}. However, their operational efficiency and effectiveness are fundamentally limited by the finite capacity of their onboard batteries. As a result, efficient energy management and route planning have become paramount, especially when robots are tasked with extended missions spanning large environments. This paper addresses the critical challenge of energy-aware route planning for battery-constrained robots operating in real-world scenarios, leveraging the presence of multiple recharging depots present within the environment. Our approach seeks to minimize the length of robot's route, while ensuring mission completion, i.e., the robot not running out of charge in a region where recharging depots are not nearby.

We approach this problem as a variant of the well-known Traveling Salesperson Problem (TSP), wherein the robot's objective is to visit all locations within the environment. These locations represent various tasks that the robot must accomplish.  Multiple recharging depots are present in the environment, and the robot has a limited battery capacity, potentially requiring multiple recharging sessions. This adds complexity beyond determining the optimal task visitation order, as it involves strategic depot visits to minimize recharging and travel time. Additionally, the problem's feasibility is not trivial, as different visitation sequences may risk the robot running out of power, far from any recharging depot. In real-world scenarios, avoiding unnecessary and excessive recharging may also be critical for the battery life, therefore we also address the related problem of minimizing the number of required recharging events. 

The challenge of planning optimal paths for mobile robots holds significance across various domains. As depicted in Figure~\ref{fig:frontpage}, precision agriculture provides a prime example. Here, mobile robots are deployed for extensive field inspections, often spanning extended timeframes. Due to limited battery capacities, these robots may necessitate multiple recharges throughout their missions. The presence of multiple strategically positioned recharging stations enhances operational flexibility, enabling the coverage of larger agricultural environments. Consequently, the pursuit of efficient routes that maximize the utility of these multiple recharging stations becomes increasingly valuable.

Beyond agriculture, this path planning challenge extends to various applications. Consider, for instance, the domain of warehouse logistics. Autonomous robots navigate extensive warehouses to fulfill orders efficiently. In this context, mapping routes that optimize the use of available charging docks can lead to considerable improvements in order fulfillment and operational efficiency.

\begin{figure}[t]
    \centering
\includegraphics[angle=0,width=0.410\textwidth]{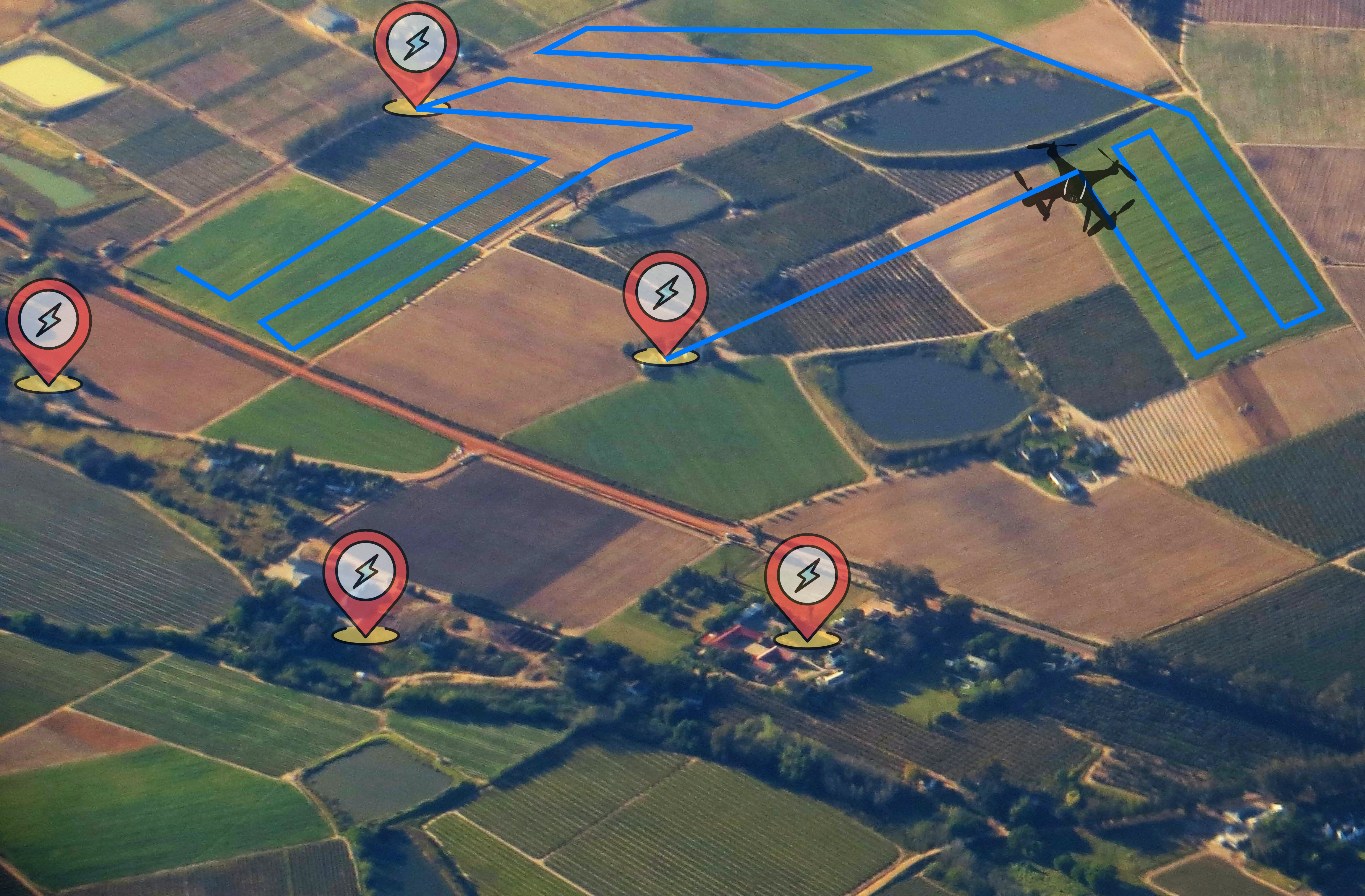}
    \caption{Application Example: An agricultural drone performs field inspections with multiple available recharging stations, showcasing the need for efficient routes with multiple recharges to cover a wide range of tasks.}
    \label{fig:frontpage}
\end{figure}

\emph{Contributions:} Our contributions in this paper are as follows:
\begin{itemize}
    \item We formally define the Multi-depot recharging TSP and the Minimum recharge cover problems.
    \item We present a $O(\log D)$ approximation algorithm to solve the problems where $D$ is the battery discharge time of the robot.
    \item We propose heuristic improvements to the approximation algorithm.
    \item We compare the performance of the algorithm with an Integer Linear Program on instances from TSPLIB library. 
\end{itemize}

\emph{Related Work:} The problem of finding paths for charge-constrained robots has been extensively studied in the literature. A closely related work, discussed in~\cite{khuller2011fill}, provides an $O\left(\frac{1+\alpha}{1-\alpha}\right)$ approximation algorithm for the problem. In this algorithm, each vertex must be within $\alpha D/2$ distance from a recharging depot, where $D$ represents the battery discharge time. As vertices move farther from the recharging depots, the value of $\alpha$ approaches one, leading to an increased approximation ratio. In contrast, our paper introduces an approximation ratio of $O(\log D)$, which remains consistent regardless of the environment but depends on the robot's battery capacity. It's important to note that these approximation ratios are not directly comparable and their relative performance depends on the specific characteristics of each problem instance. In~\cite{sundar2013algorithms} the $O\left(\frac{1+\alpha}{1-\alpha}\right)$ approximation ratio is extended to include directed graphs. 

A rooted minimum cycle cover problem is studied in~\cite{nagarajan2012approximation} that can be used to visit all the locations in the environment in minimum number of visits to one recharging depot. The authors propose a $O(\log D)$ approximation algorithm for this problem. We extend this algorithm to include multiple recharging depots and instead of cycles, we find a union of feasible cycles and paths. The $O(\log D)$ approximation ratio for rooted minimum cycle cover problem is improved to $O(\frac{\log(D)}{\log \log (D))}$ in~\cite{friggstad2014approximation}.

Battery constraints are pivotal in persistent monitoring missions with extended durations, making battery-constrained planning for persistent coverage a well-researched field. In~\cite{mitchell2015multi}, a heuristic method is presented for finding feasible plans to visit locations, focusing on maximizing the frequency of visits to the least-visited locations. In~\cite{asghar2023multi}, an approximation algorithm is proposed for a single-depot problem where multiple robots must satisfy revisit time constraints for various locations within the environment. In~\cite{hari2022bounds}, the problem of persistent monitoring with a single UAV and a recharging depot is explored using an estimate of visitable locations per charge to determine a solution. In~\cite{asghar2023risk}, the challenge of recharging aerial robots with predefined monitoring paths on moving ground vehicles is addressed.

\section{Problem Statement}
Consider an undirected weighted graph $G=(\{V\cup Q\},E)$ with edge weights $d(i,j)$ for edge $\{i,j\} \in E$, where the edge weights represent travel times for the robot. The vertex set V represents the locations to be visited by the robot, while the set $Q$ represents the recharging depots. The robot operates with limited battery capacity, and the battery discharge time is denoted as $D$. The robot's objective is to visit all the locations in $V$ while ensuring that it does not run out of battery. The formal problem statement is as follows.

%In this paper, we explore an extension of the Traveling Salesperson Problem, wherein the robot must visit all vertices within the set $V$, minimizing the overall time required, all while ensuring it maintains sufficient charge throughout its journey.

\begin{problem}[Multi-depot Recharging TSP]\label{pbm:main}
Given a graph $G=(\{V\cup Q\},E)$ with recharging depots $Q$, and a battery discharge time $D$, find a minimum length walk that visits all vertices in $V$ such that the vehicle does not stay away from the recharging depots in $Q$ more than $D$ amount of time.  
\end{problem}

The edge weights are metric, allowing us to work with a metric completion of the graph. Throughout this paper, we assume that $G$ is a complete graph. The robot requires a specified time $T$ for recharging at a recharging depot $q \in Q$. To simplify our analysis, we can transform this scenario into an equivalent problem instance where the recharge time is zero. We achieve this by adding $T/2$ to the edge weights of all edges connected to $q$ and updating the discharge time to $D+T$. Consequently, we can proceed with our analysis under the assumption that the robot's recharge time is zero. We also assume that the robot moves at unit speed, making time and length equivalent.
% \textbf{Given:}
% \begin{itemize}
%     \item An undirected weighted graph $G=(\{V\cup Q\},E)$ with edge weights $w_{ij}$ for edge $\{i,j\} \in E$.
%     \item Battery discharge time $D$ for the robot.
%     \item Set of battery recharging locations $ Q=\{q_1\,\ldots,q_M\}$.
% \end{itemize} 

%\textbf{Find:} a minimum length walk for the robot that visits all vertices in the graph without running out of charge, i.e., the robot cannot spend more than $D$ amount of time without visiting any of the depots (vertices in $Q$). 

%\textbf{Result:} There is a $O(\log{(|V|)})$ approximation algorithm to solve the problem. (There may also be a $O(\log{D})$ approximation algorithm as well.)

Note that, in contrast to the Traveling Salesperson Problem where a tour is required, we are seeking a walk within $G$ as some recharging depots may be visited more than once and some many never be visited. In this paper, we also consider a closely related problem focused on minimizing the number of visits to the recharging depots. This problem setting not only enables us to develop an algorithm for Problem~\ref{pbm:main} but is also inherently interesting, as keeping batteries in high state of charge leads to battery deterioration~\cite{takahashi2014estimating}.

\begin{problem}[Minimum Recharge Cover Problem]\label{pbm:min_recharges}
Given a graph $G=(\{V\cup Q\},E)$ with recharging depots $Q$, and battery discharge time $D$, find a walk with minimum number of recharges that visits all vertices in $V$ such that the vehicle does not stay away from the recharging depots in $Q$ more than $D$ amount of time.  
\end{problem}

%For the simplicity of analysis, we assume that the first and last vertices on the optimal solution be in $Q$. \todo{I think this should work without loss of generality.}
In the following section, we present approximation algorithms for these problems and analyze their performance. 

\section{Approximation Algorithm}

We start by observing that the recharging depots visited by a feasible solution cannot be excessively distant from each other. To formalize this, consider the graph of recharging depots denoted as $G_Q=(Q,E_Q)$, where the vertex set $Q$ represents the recharging depots and an edge $\{q_i,q_j\}\in E_Q$ if and only if $d(q_i,q_j)\leq D$. Let $\{\mathcal{C}_1,\ldots,\mathcal{C}_a\}$ be the connected components of this graph. 
\begin{observation}\label{lem:connected_components}
    If a solution to the Problems~\ref{pbm:main} and~\ref{pbm:min_recharges} contains a recharging depot from the connected component $\mathcal{C}_j$, it cannot contain a recharging depot from any other connected component of $G_Q$.
\end{observation}

Therefore, we can solve the problem by considering recharging depots from one connected component at a time, going through all connected components sequentially, and selecting the best solution.
We will need the following definition of \emph{feasible segments} in order to present the algorithms and their analysis.
\begin{definition}[Feasible Segments]
    The set $S=\{P_1,\ldots,P_a,C_1,\ldots, C_b\}$ containing paths $P_1,\ldots,P_a$ and cycles $C_1,\ldots, C_b$ in graph $G$ is called a set of feasible segments if it satisfies the following.
    \begin{enumerate}
        \item Each path and cycle in $S$ has length at most $D$.
        \item Each cycle in $S$ contains a vertex in $Q$.
        \item Each path in $S$ has both its endpoints in $Q$.
    \end{enumerate}
\end{definition}

To tackle the problem of finding minimum number of recharges to visit all vertices (Problem~\ref{pbm:min_recharges}), we first consider an interim problem, that of finding minimum number of feasible segments that cover the vertex set $V$. Then we will use the solution to this problem to construct a solution to Problems~\ref{pbm:main} and~\ref{pbm:min_recharges}.

\subsection{Minimum Segment Cover Problem}
In~\cite{nagarajan2012approximation}, the single depot problem of finding the minimum number of length-constrained rooted cycles is addressed with an $O(\log D)$ approximation algorithm. We extend this algorithm to find segments instead of cycles in the multi-depot scenario. 

Let the farthest distance to be traveled from the closest depot to any vertex be defined as
\begin{equation*}
    \Delta = \max_{v\in V} \min_{q\in Q}d(v,q).
\end{equation*}

Note that $\Delta \leq D/2 $ in a feasible problem instance.\footnote{This condition is necessary but not sufficient since the closest depots to two different vertices may be farther than $D$ without any depot in between.} Also let $d(Q,v)$ be the distance from vertex $v\in V$ to its nearest recharging depot, i.e., $d(Q,v) = \min_{q\in Q} d(q,v)$.
Now define the slack as $\delta = \frac{D}{2} -\Delta + 1$.

We can now partition the vertices in $V$ by defining the vertex set $V_j$ for $j\in\{0,\ldots,t\}$ where $t= \lceil \log ({D/2\delta}) \rceil$ 
\begin{equation}
\label{eq:vertex_sets}
    V_j =\begin{cases}
        \{v: \frac{D}{2}-\delta < d(Q,v)\leq \Delta \} & \text{if } j=0\\
        \{v: \frac{D}{2}-2^j\delta < d(Q,v)\leq \frac{D}{2}-2^{j-1}\delta \} & \text{otherwise}.
    \end{cases}
\end{equation}

\begin{algorithm}
\DontPrintSemicolon
	%\openup -.1em
\caption{\textsc{MinSegmentCover}}
\label{alg:minrechargecover}
\KwIn{Graph $G=(V\cup Q,E)$ with weights $w_{ij}, \forall\{i,j\}\in E$, discharging time $D$}
\KwOut{A set of feasible segments $S$ covering all the vertices in $V$}
  \vspace{0.2em}
  \hrule
  $S\leftarrow \emptyset$ \;
	Define the vertex sets $V_j$ using Equation~\eqref{eq:vertex_sets} \;
    \For {$j=0$ to $\lceil\log(D/2\delta)\rceil$} 
    {
    $\Pi_j \leftarrow$ \textsc{UnrootedPathCover} $(V_j,2^j\delta-1 )$ \;\label{algln:paths}
    \ForEach {$p$ in $\Pi_j$}
    {
    Connect each endpoint of $p$ to its nearest depot\;
    Add the resulting segment to $S$\;
    }
    }
\end{algorithm}
For each partition $V_j$, Algorithm~\ref{alg:minrechargecover} finds the minimum number of paths covering the vertices in $V_j$ such that the length of these paths is no more than $2^j\delta-1$ (line~\ref{algln:paths}). A $3$--approximation algorithm for finding this path cover is presented in~\cite{arkin2006approximations}. The endpoints of each path are then connected to their nearest recharging depots. The following Lemma establishes the approximation ratio of Algorithm~\ref{alg:minrechargecover}.

\begin{lemma}\label{lem:min_recharge_cover}
    If the edge weights of graph are integers, Algorithm~\ref{alg:minrechargecover} is a $O(\log D)$ approximation algorithm to find a minimum cardinality set of feasible segments that covers all vertices in $V$.
\end{lemma}
\begin{proof}
The segments in $S$ are feasible because each path in $\Pi_j$ in line~\ref{algln:paths} has length $2^j\delta-1$ and since the endpoints of this path are in $V_j$, the closest depots cannot be more than $D/2-2^{j-1}\delta$ by Equation~\eqref{eq:vertex_sets}. Connecting two closest depots to the endpoints results in the total length $2^j\delta-1 +2(D/2-2^{j-1}\delta) = D-1$ for each segment. (For $j=0$ the path length is at most $\delta-1$ and the closest depots to endpoints are less than $\Delta$ distance away resulting in the maximum segment length $\delta-1+2\Delta = D/2+\Delta\leq D$.)

Let $S^*$ represent an optimal solution with the minimum number of feasible segments. Consider any segment of $S^*$, and let $s_j$ represent the path induced by that segment on the set $V_j$. Since every vertex in $V_j$ is at least $\frac{D}{2}-2^j\delta +1$ distance from any $q\in Q$, the length of $s_j$ is at most $D-2(\frac{D}{2}-2^j\delta +1) = 2^{j+1}\delta-2$. We can split this path $s_j$ to get two paths of length at most $2^j\delta-1$. If we do this with every segment in $S^*$, we get at most $2|S^*|$ paths that cover $V_j$ and are of length $2^j\delta-1$. Therefore running the 3 approximation algorithm for the unrooted path cover problem with bound $2^j\delta-1$ gives us at most $6|S^*|$ paths. Moreover, since the number of subsets is $t= \lceil \log ({D/2\delta}) \rceil$, the number of segments returned by algorithm~\ref{alg:minrechargecover} is $|S|\leq 6\log(D)|S^*|$.
\end{proof}

\begin{remark}
\label{rem:logn}
    This segment cover problem can also be posed as a set cover problem where the collection of sets consists of all the possible feasible segments, and we have to pick the minimum number of subsets to cover all the vertices. This problem can be solved greedily by using an approximation algorithm to solve the Orienteering problem between all pairs of depots at each step. This results in a $O(\log(n))$ approximation ratio for the problem of finding minimum number of feasible segments to cover $V$. Note that the two approximation ratios $O(\log (n))$ and $O(\log(D))$ are incomparable in general, however, for a given robot with fixed $D$, increasing the environment size would not affect the approximation ratio in case of $O(\log D)$ algorithm.
\end{remark}
%\textcolor{blue}{If we try to do the analysis of the $\log D$ algorithm for the case with non integer edge weights, we get a $\log n$ ratio, which is easier to prove using the set cover greedy algorithm. This makes sense because if edge weights can be any real/rational number, we can get an equivalent problem with very small value of $D$ and get a different approximation ratio for the number of segments.}

Note that the segments returned by Algorithm~\ref{alg:minrechargecover} may not be traversable by the vehicle in a simple manner, since the segments may not be connected. In the following subsection we present an algorithm to construct a feasible walk in the graph given the set of segments.

\subsection{Algorithm for the feasible walk}
%Let the set of segments returned by Algorithm~\ref{alg:minrechargecover} be $S$.

If some segments are close to each other, the robot can visit these segments before moving on to farther areas of the environment. To formalize this approach, we need the following definition of Neighboring Segments. 
\begin{definition}[Neighboring Segments]
\label{def:neighboring_seg}
Any two segments $s_i$ and $s_j$ in the set of feasible segments $S$ are considered neighboring segments if there exist depot vertices $q_i$ in $s_i$ and $q_j$ in $s_j$ such that either 1) $q_i= q_j$, or 2) $d(q_i,q_j)\leq D$, or 3) it is possible to traverse from $q_i$ to $q_j$ with at most one extra recharge, i.e., there exists $q\in Q$ such that $d(q_i,q)\leq D$, and $d(q,q_j)\leq D$.
% \begin{enumerate}
%     \item $q_i= q_j$, or
%     \item $d(q_i,q_j)\leq D$, or
%     \item it is possible to traverse from $q_i$ to $q_j$ with at most one extra recharge, i.e., there exists $q\in Q$ such that $d(q_i,q)\leq D$, and $d(q,q_j)\leq D$.
% \end{enumerate}
\end{definition}
We will refer to the set of neighboring segments $S_i$ as the maximal set of segments such that every segment in $S_i$ is a neighbor of at least one other segment in $S_i$.

The approximation algorithm is outlined in Algorithm~\ref{alg:approx}. In Line~\ref{algln:minsegment}, we construct a set $S$ of feasible segments. Then we create a graph $G_S$ where each vertex corresponds to a set of neighboring segments in $S$ and the edge weights represent the minimum number of recharges required to travel between these sets. This graph is both metric and complete. A TSP solution is then determined for $G_S$, and the set of neighboring segments are visited in the order given by the TSP. Inside each set of neighboring segments, the segments are traversed as explained in the following lemma.

\begin{lemma}
\label{prop:nbr_segments}
    Given a set of neighboring segments $S_i$ with $|S_i|$ segments, a traversable walk can be constructed visiting all the vertices covered by $S_i$ in at most $6|S_i|$ recharges.
\end{lemma}
\begin{proof}
We can construct a Minimum Spanning Tree $\mathcal{T}$ of the depots in $S_i$ and traverse that spanning tree to visit all the recharging depots in $S_i$, and therefore all the vertices covered by $S_i$. Note that we cannot use the usual short-cutting method for constructing Travelling Salesman Tours from Miniumum Spanning Trees because that may result in edges that are longer than $D$. Since each segment in $S_i$ can be a path with a recharging depots at each end, these segments contribute $2|S_i|$ recharging depots. Also, in the worst case, each neighboring segment is one recharge away from its closet neighbor resulting in $|S_i|$ connecting depots. Thus we have a Minimum Spanning Tree with at most $3|S_i|-1$ edges (with a recharge required after traversing each edge), and in the worst case we may have to traverse each edge twice resulting in at most $6|S_i|$ recharges.
\end{proof}

%If any two segments in $S$ share a recharging depot in $Q$, we call them connected segments. The algorithm is as follows.
%\todo{When making vertices, join any two connected segments that have any two recharging depots within two hops. Then show using triangle inequality that any vertex being served by a depot in the optimal solution, that depot cannot serve any other vertex in any other vertex $u_i$. Therefore in the optimal solution those different depots (same number) need to be connected together using TSP. We are also connecting them using TSP with different but related edge lengths. Number of segments themselves is handled by Lemma 2.3. }

%Now that we have a procedure to find a walk for visiting all the vertices in each set of neighboring segments, we propose the following method to find a walk that visit all the sets of neighboring segments without running out of charge.

%Note that the graph $G=(U,E_S)$ constructed in Lines~\ref{algln:1} through~\ref{algln:2} is complete metric graph.\todo{complete the description of Algorithm~\ref{alg:approx}.}

\begin{algorithm}
\DontPrintSemicolon

    \caption{\textsc{ApproximationAlgorithm}}
    \label{alg:approx}
    \KwIn{Graph $G=(V\cup Q,E)$ with weights $w_{ij}, \forall\{i,j\}\in E$, discharging time $D$}
    \KwOut{A feasible walk covering all $v\in V$}
    \vspace{0.2em}
  \hrule
    %\tcp{Step 1: Construct a weighted graph GS}
    $S\leftarrow$\textsc{MinSegmentCover}$(G,D,Q)$\;\label{algln:minsegment}
    $\{S_1,\ldots,S_k\}\leftarrow$ Neighboring segments from $S$\;
    $U \leftarrow$ Create an empty set of vertices\;
    $E_S \leftarrow$ Create an empty set of edges\;
    
    \ForEach{set of neighboring segments $S_i \in S$}{
        Add a vertex $u_i$ to $U$\;\label{algln:1}
    }
    
    \ForEach{$u_i, u_j \in U$, $u_i \neq u_j$}{
        $d_Q(S_i, S_j) \leftarrow$ minimum number of recharges required to go from any depot in $S_i$ to any depot in $S_j$\;
        Add edge $\{u_i, u_j\}$ with weight $d_Q(S_i, S_j)$ to $E_S$\;\label{algln:2}
    }
    
    %\tcp{Step 2: Find a TSP using Christofides' algorithm}
    
    $\texttt{TSP} \leftarrow$ Apply Christofides' algorithm to $G_S=(U,E_S)$\;
    
    %\tcp{Step 3: Visit neighboring segment sets}
    
    \ForEach{$S_i$ in the order given by $\texttt{TSP}$}{
        $\texttt{MST}_i\leftarrow$ Minimum Spanning Tree of depots in $S_i$\;
        Traverse segments connected to depots in $\texttt{MST}_i$
    }
    
\end{algorithm}

%\todo{should add figures to better demonstrate the algorithm}
Now we show that the cost of the TSP on $G_S$ is within a constant factor of the optimal number of recharges.
\begin{lemma}
\label{prop:TSP_seg}
The cost of the TSP solution to $G_S$ is not more than the $4.5$ times the optimal number of recharges required to visit $V$.
\end{lemma}
\begin{proof}
Any optimal solution must visit the vertices covered by $S_i$ at least once for all $i$. Let $q_i^{*}$ be a recharging depot visited by the optimal solution that is within $D/2$ distance of some vertex $v_i$ in $S_i$. 
\begin{claim}
    There exists a unique $q_i^{*}$ for every set of neighboring segment $S_i$. Also the minimum number of recharges required to go from $S_i$ to $S_j$ is $d_Q(S_i,S_j)\leq 3d_Q(q_i^*,q_j^*)$.
\end{claim}
\begin{proof}
    For a vertex $v_i$ visited be one of the segments in $S_i$, the recharging depot $q_i^*$ always exists within $D/2$ of $v_i$ because if there is no such depot, the vertex $v_i$ cannot be visited by the optimal solution. Also for any $i,j$ such that $i\neq j$, $q_i^* \neq q_j^*$. If this is not true, i.e., $q_i^* = q_j^* = q^*$ , then the distance between $v_i\in S_i$  and $v_j\in S_j$ is at most $D$ due to triangle inequality. Let $q_i$ be the recharging depot associated with the segment in $S_i$ that visits $v_i$ and is within $D/2$ distance of $v_i$ (if the segment is a cycle, $q_i$ is the depot in that cycle and if the segment is a path, $q_i$ is the closer of the two depots to $v_i$). Similarly let $q_j$ be the recharging depot associated with $v_j$. Then $d(q_i,q^*)\leq D$ and $d(q^*,q_j)\leq D$. Therefore we can go from $q_i$ in $S_i$ to $q_j$ in $S_j$ with one recharge at $q^{*}$. This means that $S_i$ and $S_j$ are not two different sets of neighboring segments by Definition~\ref{def:neighboring_seg} which results in a contradiction. Therefore  $q_i^* \neq q_j^*$ for $i\neq j$. This implies that the minimum number of recharges required to go from $q_i$ to $q_j$, i.e., $d_Q(q_i,q_j)$ is at most $3d_Q(q_i^*,q_j^*)$. Therefore $d_Q(S_i,S_j)\leq d_Q(q_i,q_j)\leq 3d_Q(q_i^*,q_j^*)$.
\end{proof} 

Since any optimal solution must visit each $S_i$ at least once, the number of recharges required by the optimal solution must be at least equal to the cost of the optimal TSP solution visiting $q_i^*$ for all set of neighboring segments $i$. Hence the cost of the TSP solution on $G_S$ returned by the Christophedes' approximation algorithm is within $4.5$ times the number of recharges required by the optimal solution.
\end{proof}

We can now use Lemmas~\ref{prop:nbr_segments} and~\ref{prop:TSP_seg} to give abound on the number of visits to the recharging depots made by the walk returned by Algorithm~\ref{alg:approx}. Note that, if we have a better algorithm for the minimum segment cover problem, the approximation ratio for Problem~\ref{pbm:min_recharges} also improves.

\begin{theorem}
\label{thm:min_recharge}
    Given a $O(h(G,D))$ approximation algorithm for the problem of finding a minimum cardinality set of feasible segments covering the vertices in $V$, we have a $O(ch(G,D))$ approximation algorithm for Problem~\ref{pbm:min_recharges}, where $c$ is a constant.
\end{theorem}
\begin{proof}
Let the optimal number of visits to the recharging depots be $\texttt{OPT}$. By Propositions~\ref{prop:nbr_segments} and~\ref{prop:TSP_seg}, the total number of recharges required will be at most $6|S|+4.5\texttt{OPT}$. Since $|S|\leq O(h(G,D))\texttt{OPT}$, we have a $ O(6h(G,D)+4.5)$ approximation ratio for Problem~\ref{pbm:min_recharges}.
\end{proof}

By Lemma~\ref{lem:min_recharge_cover} and Remark~\ref{rem:logn}, $h(G,D) = O(\min\{\log(D),\log(n)\})$. Therefore, we have a $O(\min\{\log(D),\log(n)\})$ approximation algorithm for Problem~\ref{pbm:min_recharges}.

We now show that the approximation algorithm for Problem~\ref{pbm:min_recharges} is also an approximation algorithm for Problem~\ref{pbm:main}. 
In the following proof, the feasible segment $S$ may also contain segments that do not visit any other vertices and are just direct paths connecting one recharging depot to another.
\begin{theorem}
\label{thm:main}
    A $O(f(G,D))$ approximation algorithm for Problem~\ref{pbm:min_recharges} is also a $O(4 f(G,D))$ approximation algorithm for Problem~\ref{pbm:main}.
\end{theorem}
\begin{proof}
Let an optimal solution to Problem~\ref{pbm:main} correspond to the feasible segment set $\Tilde{S}$. Also let the total length of the segments in $\Tilde{S}$ be $\ell^*$ which is the cost of the optimal solution. Note that in that set at least half the segments have length greater than or equal to $D/2$. Otherwise, there are at least two consecutive segments in the optimal solution of length less than $D/2$ and therefore one recharging visit can be skipped with no increase the cost. Therefore, 
\begin{align*}
    \ell^* \geq \frac{D}{2}\frac{|\Tilde{S}|}{2} 
\end{align*}
Now consider the approximate solution to Problem~\ref{pbm:min_recharges}, corresponding to feasible segment set $S$. Note that this feasible segment set now may contain some segments that are connecting vertices in $Q$ without visiting any other vertex in between. Also let $S^*$ be the feasible segment set corresponding to the optimal solution of Problem~\ref{pbm:min_recharges}. Then $|\Tilde{S}|\geq |S^*| $, and by the approximation ratio, $|S|\leq f(G,D)|S^*|$. Therefore, the total length of the segments in $S$ or the cost of our solution is given by
\begin{align*}
    \ell(S) = D|S|\leq Df|S^*|&\leq Df|\Tilde{S}| 
    \leq 4f \ell^*.
\end{align*}
\end{proof}
The runtime of Algorithm~\ref{alg:approx} comprises the runtime of Algorithm~\ref{alg:Heuristic} in addition to the time taken to construct the graph $G_S$ and execute a TSP algorithm on it. As $G_S$ contains at most $O(|V|)$ nodes, and we employ Christofides' algorithm to solve the TSP, this runtime is polynomial. Furthermore, Algorithm~\ref{alg:minrechargecover} solves the unrooted path cover problem for each partition $V_i$. Since the approximation algorithm for unrooted path cover from~\cite{arkin2006approximations} operates in polynomial time, the overall runtime of Algorithm~\ref{alg:approx} is polynomial.

In the next section we suggest some heusristic improvements to the approximation algorithm.
%\todo{We can also consider the main problem to be the min max weighted latency problem with multiple recharging depots, and then we can use this result to give a $O(\log D \log \rho)$ approximation algorithm for that problem where $\rho$ is the ratio between the maximum and minimum vertex weights. Maybe this is the addition for journal paper.}

%\section{Bicretiria Approximation Idea}
%What if instead of covering the whole graph, we have a problem of maximizing whatever we visit (reward) within a given number of recharges, (or maybe a given recharge budget if different recharge depots have different costs), then the submodular orienteering stuff would give constant factor approximation for the feasible segment set in terms of the reward gathered. However, when we make a traversable tour out of it, we will need more recharges, hence violating the recharging budget resulting in a bicriterion approximation. We may need to run a binary serach on the number of recharges allowed in the first step.
\begin{figure*}[t]
  \centering
  \begin{subfigure}[b]{0.2\textwidth}
    \includegraphics[width=\textwidth]{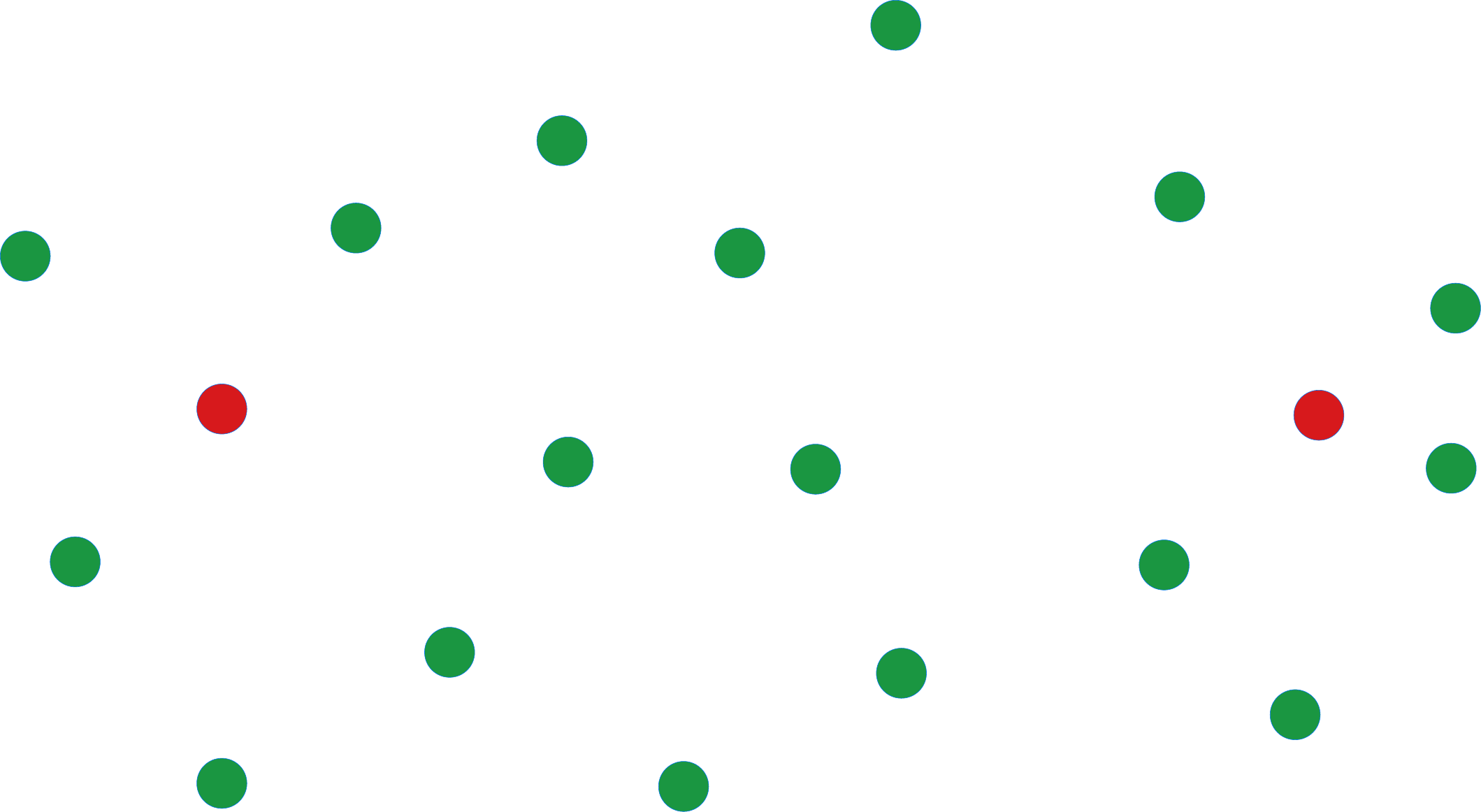}
    \caption{}
    \label{fig:1a}
  \end{subfigure}\hspace{5em}
  \begin{subfigure}[b]{0.2\textwidth}
    \includegraphics[width=\textwidth]{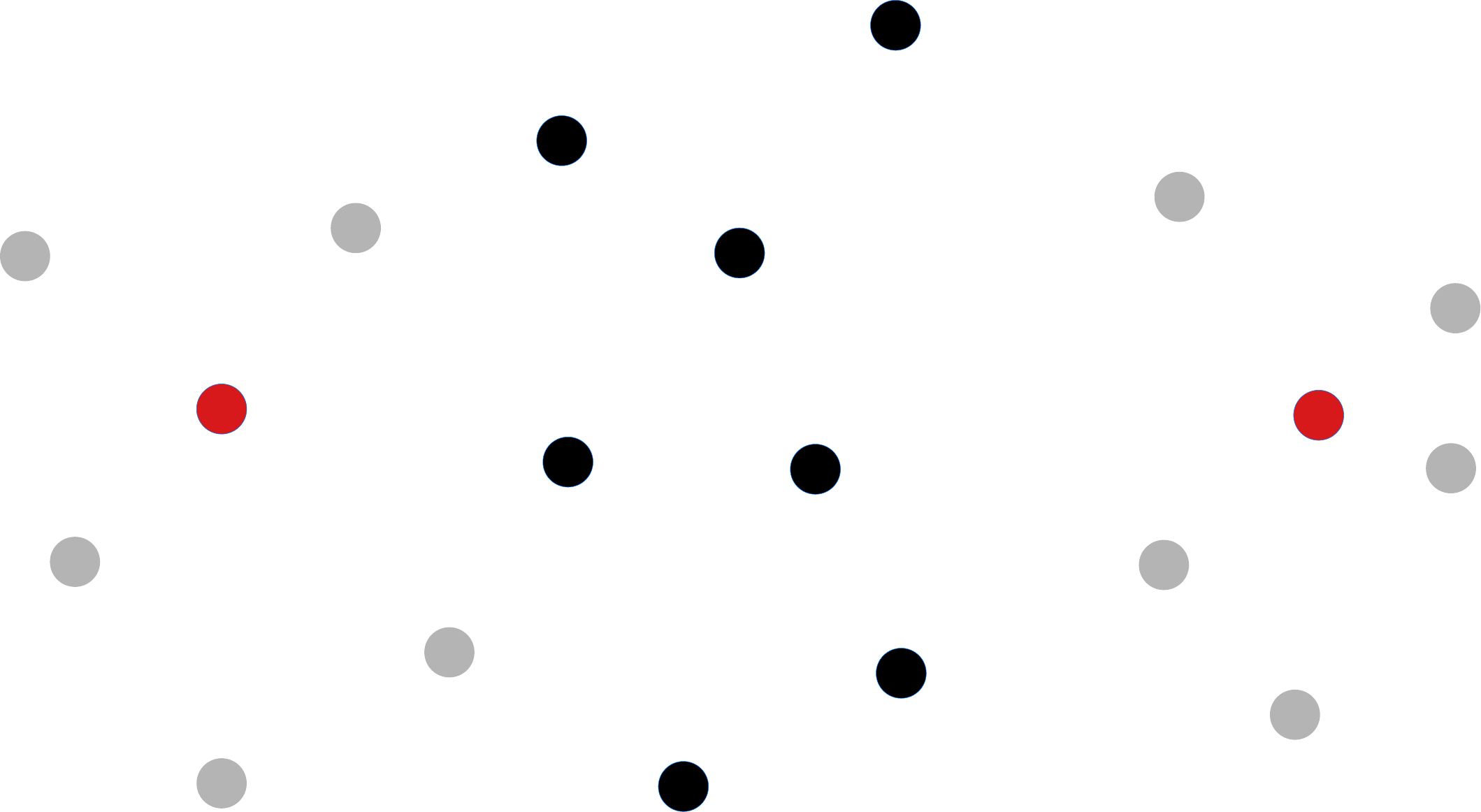}
    \caption{}
    \label{fig:1b}
  \end{subfigure}\hspace{5em}
  \begin{subfigure}[b]{0.2\textwidth}
    \includegraphics[width=\textwidth]{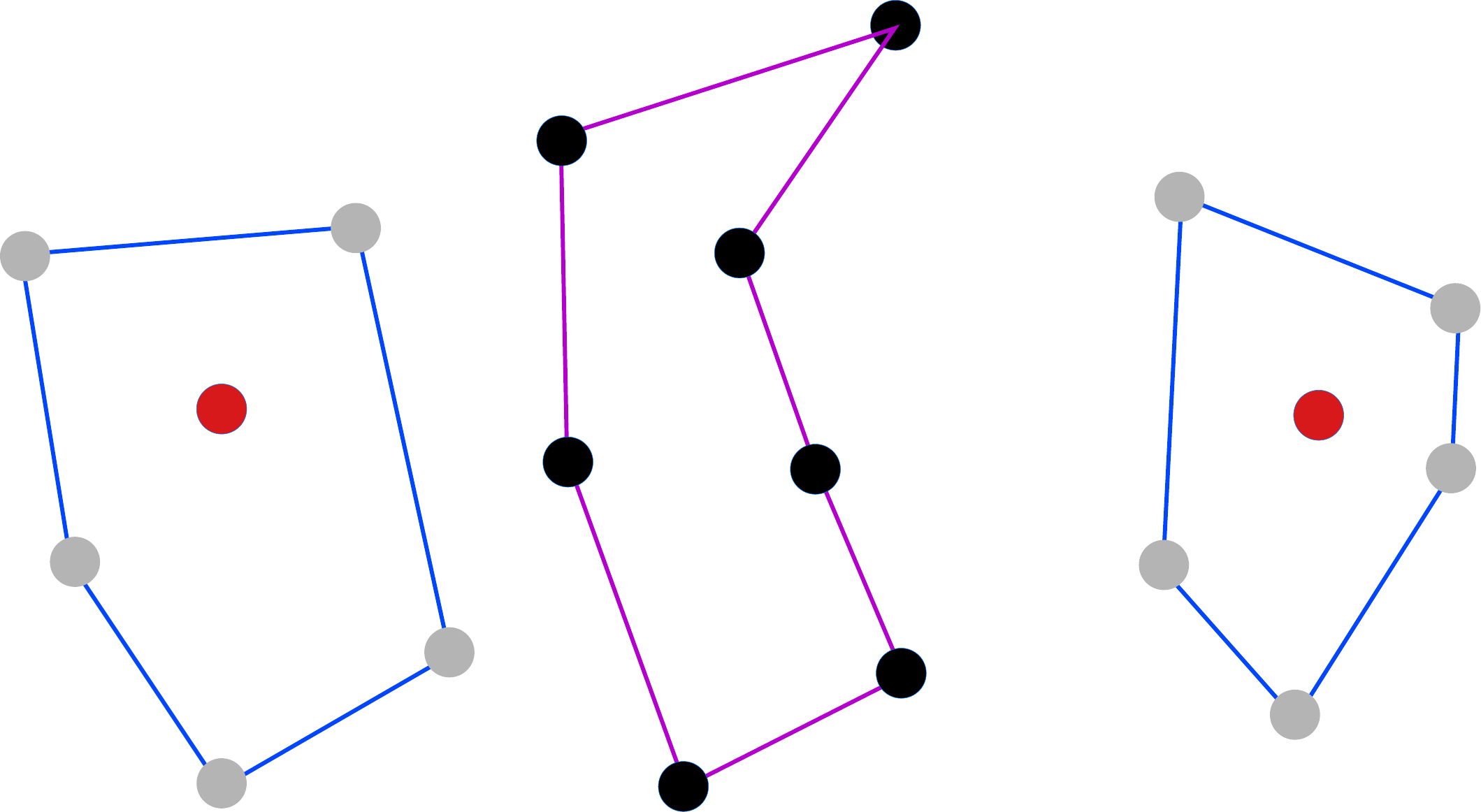}
    \caption{}
    \label{fig:1c}
  \end{subfigure}
  
  \begin{subfigure}[b]{0.2\textwidth}
    \includegraphics[width=\textwidth]{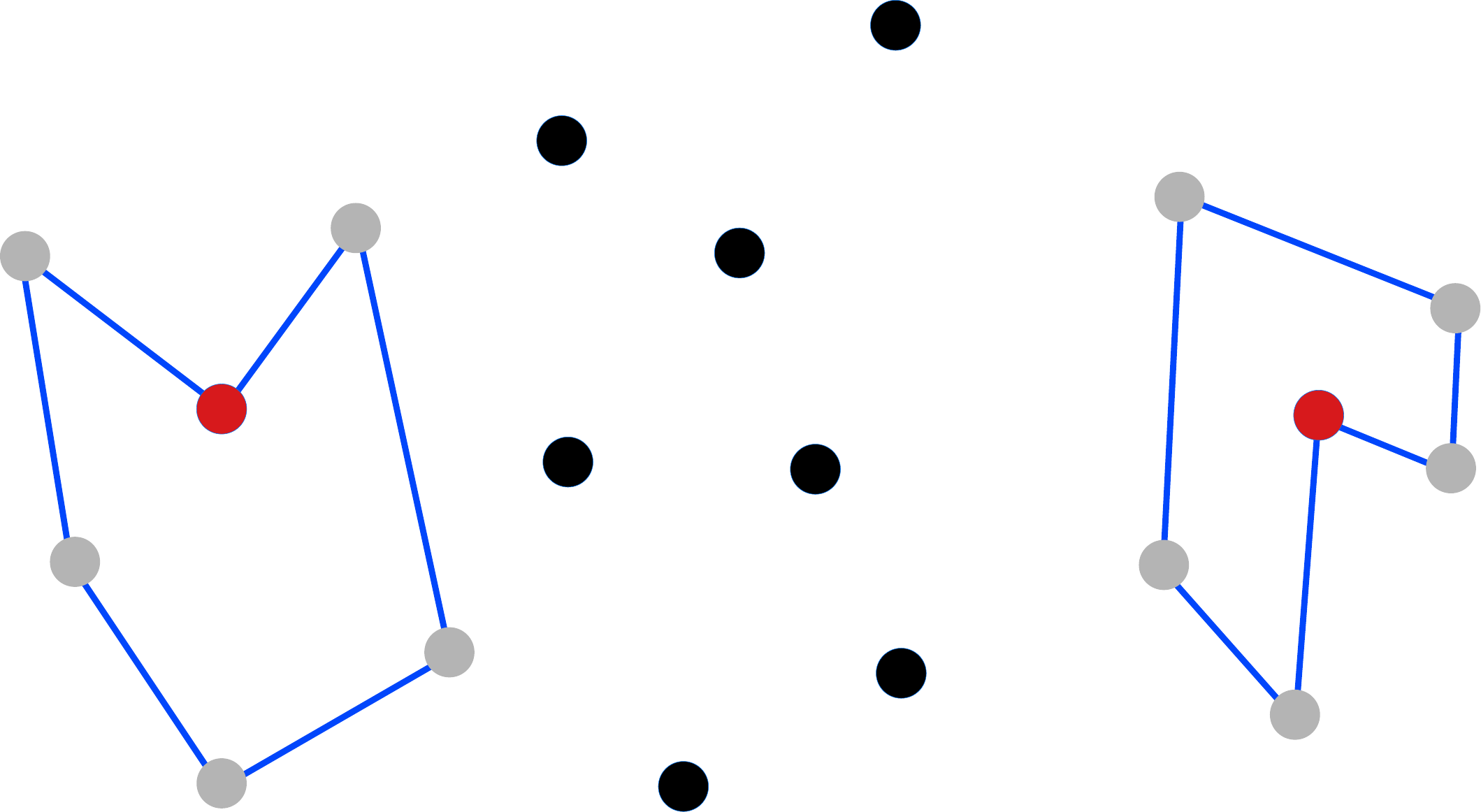}
    \caption{}
    \label{fig:2a}
  \end{subfigure}\hspace{5em}
  \begin{subfigure}[b]{0.2\textwidth}
    \includegraphics[width=\textwidth]{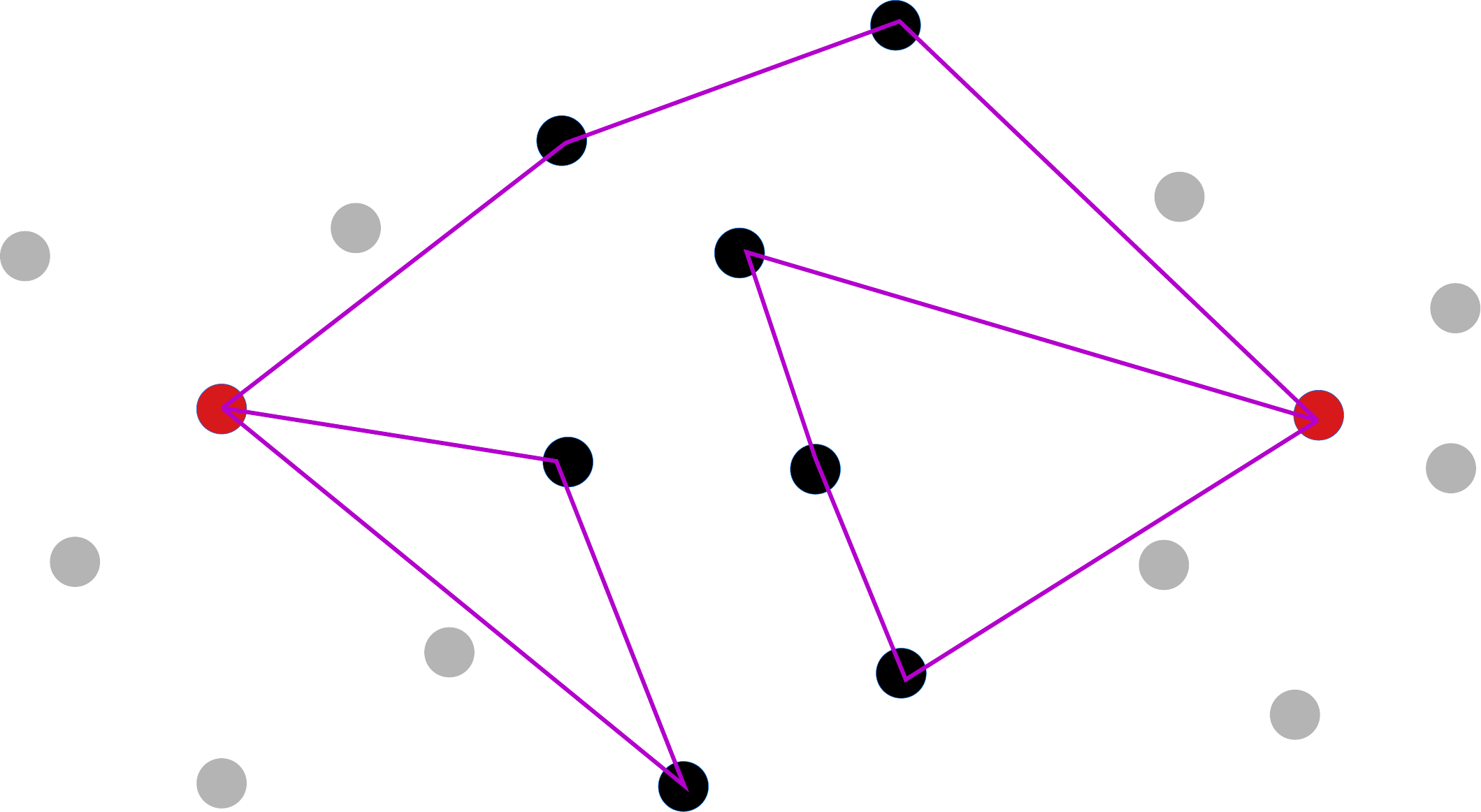}
    \caption{}
    \label{fig:2b}
  \end{subfigure}\hspace{5em}
  \begin{subfigure}[b]{0.2\textwidth}
    \includegraphics[width=\textwidth]{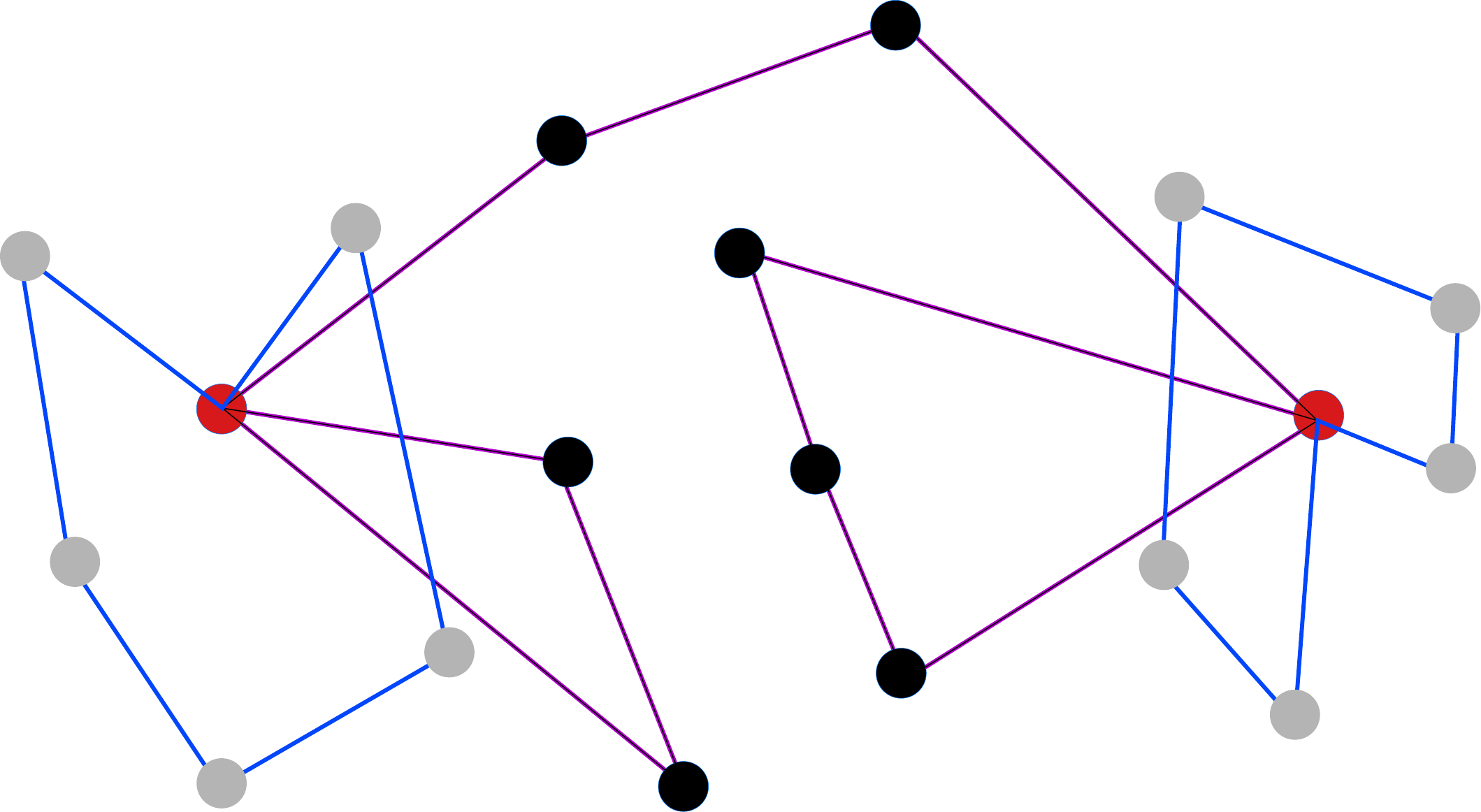}
    \caption{}
    \label{fig:2c}
  \end{subfigure}
  
  \caption{An example execution of Algorithm~\ref{alg:Heuristic} is shown. (a) displays the problem instance with two depots highlighted in red, and the vertices to be visited marked in green. The vertices are partitioned based on their distance from the depots in (b). Black vertices belong to $V_0$, while grey vertices belong to $V_1$. In this example, $b=1$ in Line~\ref{algln:b} of the algorithm. The spanning forest and subsequent tours for each $V_i$ are found, and tours for $k=1$ for $V_0$ and $k=2$ for $V_1$ are displayed in (c). The depots are then inserted into the tours of $V_1$ and $V_0$ to make them feasible in (d) and (e), respectively. The final route is presented in (e), where the cycles of one depot are traversed first before moving on to the next depot.}
  \label{fig:example}
\end{figure*}

\section{Heuristic Improvements}
In this section, we introduce heuristic improvements to the approximation algorithm presented in the previous section. While these improvements do not impact the algorithm's approximation ratio, they significantly enhance solution quality in practice. One of the advantages of presenting algorithms and their analysis as done in the previous section is that when we enhance a specific aspect of the overall algorithm, we can assess its effect on the entire algorithm. For instance, if an algorithm with an improved approximation ratio for the minimum segment cover supersedes Algorithm~\ref{alg:minrechargecover}, Theorems~\ref{thm:min_recharge} and~\ref{thm:main} enable us to utilize that algorithm for achieving better approximation ratios for Problems~\ref{pbm:min_recharges} and~\ref{pbm:main}.

The heuristic algorithm for solving Problems~\ref{pbm:main} and~\ref{pbm:min_recharges} is presented in Algorithm~\ref{alg:Heuristic}. When partitioning the vertices in $V$, instead of choosing each partition separately, we bundle $b$ partitions with vertices closer to each other in Line~\ref{algln:b}. We can go through all choices of $b$ between $1$ and $\lceil\log(D/2\delta)\rceil$ and pick the best solution to keep the approximation guarantee. In our experiments, we observed that larger values of $b$ tend to yield superior solutions.

In the previous section's approximation algorithm, Line~\ref{algln:paths} of Algorithm~\ref{alg:minrechargecover} is used to determine the minimum number of paths covering a partition, followed by connecting the endpoints of those paths to the nearest depots to create segments. The algorithm for finding the minimum number of paths, inspired by~\cite{arkin2006approximations}, finds a tour of different connected components (using Minimum Spanning Forest in Line~\ref{algln:mst}) and then splits each tour into appropriately sized paths.

In Algorithm~\ref{alg:Heuristic}, when a tour $T$ of a connected component is found in Line~\ref{algln:tour}, it diverges from the previous approach. Instead of breaking the tour into short segments, we utilize the $\textsc{InsertDepots}$ function to traverse the tour, inserting recharging depots wherever necessary to complete it. The tour is only split when it is impossible to traverse from one vertex to the next with only one recharge in between. The resulting segments are then traversed using an ordering given by the TSP of the first and last depots of each segment. A sample execution of the algorithm is shown in Figure~\ref{fig:example}.

\begin{algorithm}
\DontPrintSemicolon
	%\openup -.1em
\caption{\textsc{HeuristicAlgorithm}}
\label{alg:Heuristic}
\KwIn{Graph $G=(V\cup Q,E)$ with weights $w_{ij}, \forall\{i,j\}\in E$, discharging time $D$}
\KwOut{A walk $W$ covering all $v\in V$}
  \vspace{0.2em}
  \hrule
  $W\leftarrow \emptyset$ \;
  $S\leftarrow \emptyset$\;
  $b\leftarrow $ choose an integer between $1$ and $\lceil\log(D/2\delta)\rceil$\;
  \For {$j=0:b:\lceil\log(D/2\delta)\rceil$} 
    {
    $\mathcal{V}\leftarrow \{V_j,\ldots,V_{b-1}\}$\;\label{algln:b}
    \For {$k=1$ to $|\mathcal{V}|$}{
     $\mathcal{F}_k\leftarrow$ Minimum Spanning Forest of $\mathcal{V}$ with $k$ components \;\label{algln:mst}
     $R_k\leftarrow\emptyset$\;
     \ForEach{component $H$ in $\mathcal{F}_k$}{
     Find a tour $T$ of vertices in $H$\;\label{algln:tour}
     $\bar{T} \leftarrow$ \textsc{InsertDepots}$(T,G,D)$\;
     $R_k\leftarrow R_k\cup \bar{T}$\;
     }
     }
     Choose $R$ from $\{R_1, \ldots, R_{|\mathcal{V}|}\}$ with fewest recharges\;
     $S\leftarrow S\cup R$\;}
     $\texttt{TSP}\leftarrow$ TSP of endpoint depots in $S$\;
     $W\leftarrow$ Traverse segments in $S$ in $\texttt{TSP}$ order
\end{algorithm}

\section{Simulation Results}
\begin{figure}[h]
  \centering
  \subfloat[ILP solution with cost $491$ and $10$ recharges.]{\includegraphics[width=0.20\textwidth]{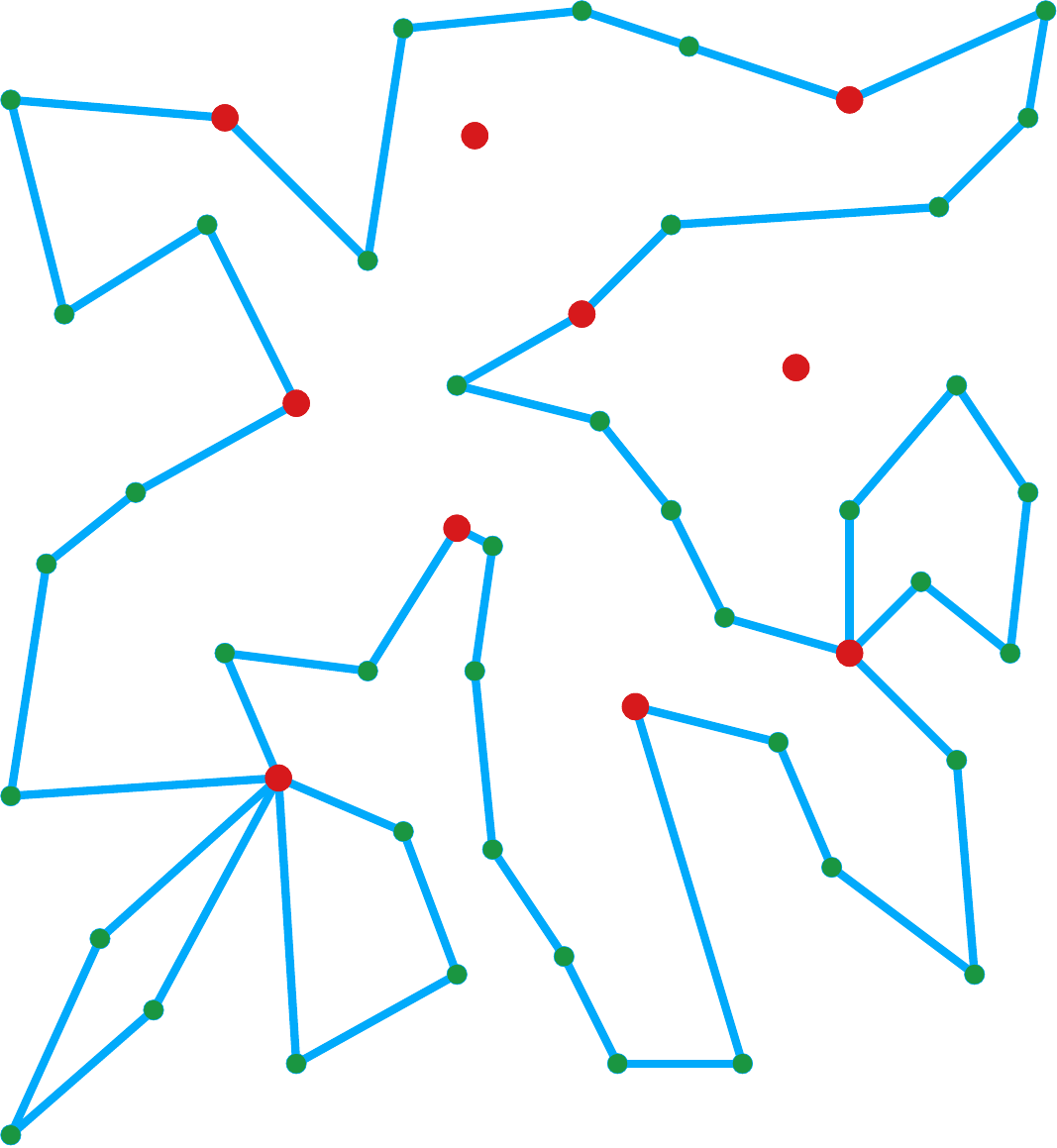}}
  \hfill
  \subfloat[Algorithm~\ref{alg:approx} solution with cost $609$ and $14$ recharges.]{\includegraphics[width=0.20\textwidth]{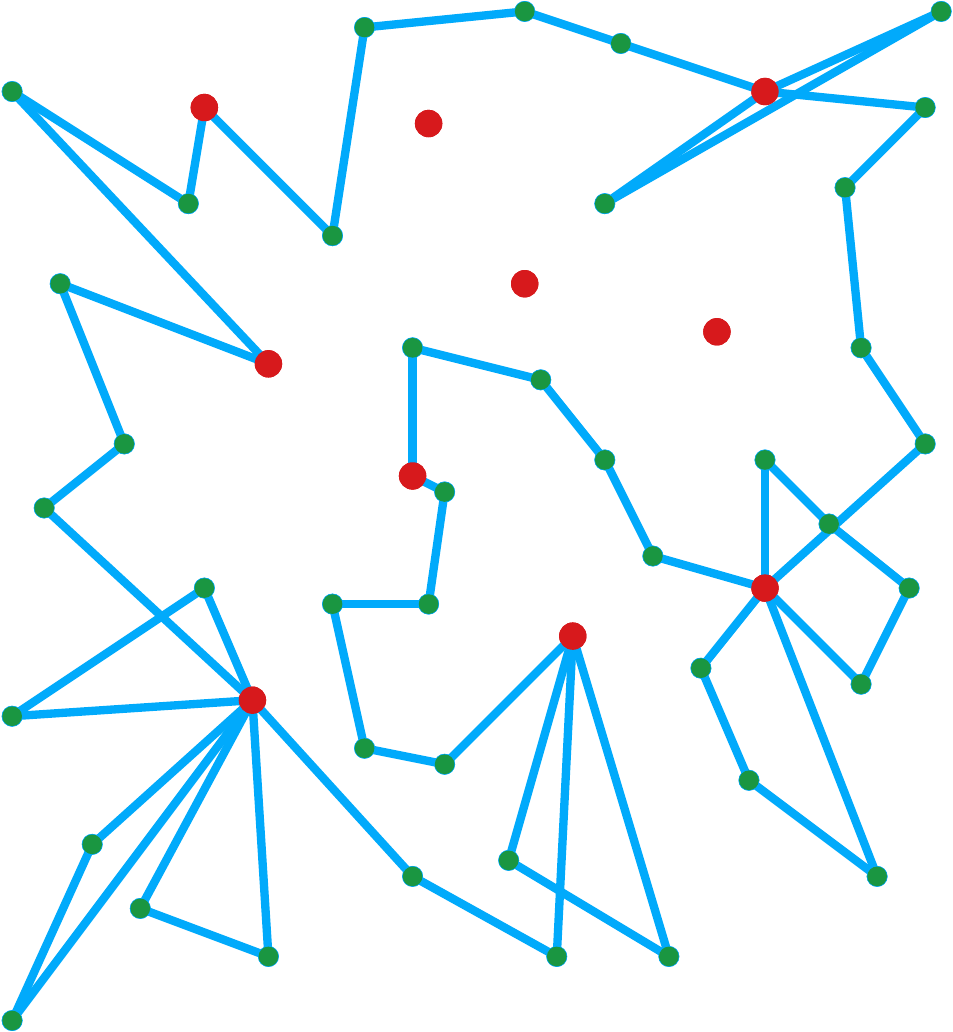}}
  \caption{Comparison of the two solutions obtained on the problem instance `eil51' with $10$ recharging locations (shown in red), $41$ vertices to be visited (shown in green), and $D=50$.}
  \label{fig:1}
\end{figure}
In this section, we compare the performance of the proposed algorithm with an Integer Linear Program (ILP). The ILP provides an optimal solution, allowing us to evaluate the cost of the solutions against the optimal cost.

To create problem instances for evaluation, we utilized the Capacitated Vehicle Routing Problem library from TSPLIB~\cite{reinelt1991tsplib}, a well-recognized library of benchmark instances for combinatorial optimization problems. It's worth noting that this library primarily includes instances with a single depot and therefore, we needed to select both the depot set $Q$ and the battery discharge time $D$ to construct a problem instance suitable for our specific problem. 

For each instance from TSPLIB with $n$ vertices (number of vertices is written with instance name in Table~\ref{tab:comparison}), we generate multiple instances by varying the number of depots, denoted as $m=|Q|$, selected from the $n$ locations, and by using different values of $D$. The run-time and the cost (length of the returned route) of the ILP and our algorithm are presented in Table~\ref{tab:comparison}.
For instances with more than $30$ vertices, the ILP did not terminate within a timeout of $5$ minutes, and we report the best solution found within that time in the table. The highlighted numbers indicate the cost of the optimal solution. Since, we are solving an NP-hard problem using ILP, as expected, the run times of the proposed algorithm are much faster than that of ILP for large instances. In fact for the instance with $262$ vertices, the ILP was not able to find any feasible solution within $5$ minutes. Despite low run times, the cost of the solutions of the proposed algorithm are within $31\%$ of that of ILP on average, and within $1.53$ factor of the ILP solution in the worst case. Figure~\ref{fig:1} shows the routes returned by both the algorithms on the Problem instance `eil51'. 

\begin{table}[t]

\caption{Comparison of ILP and Algorithm~\ref{alg:Heuristic} for TSPLIB Instances}
\centering
\setlength{\tabcolsep}{1pt}
\renewcommand{\arraystretch}{1.3}
\begin{tabular}{c c c c c c}
\hline
\multirow{2}{*}{Instance} & \multirow{2}{*}{\begin{tabular}[c]{@{}c@{}}$(m, D)$\end{tabular}} & \multicolumn{2}{c}{ILP} & \multicolumn{2}{c}{Algorithm~\ref{alg:Heuristic}} \\ %\cline{3-6} 
 &  & Runtime (s) & Cost & Runtime (s) & Cost \\ \hline
 \multirow{2}{*}{eil23} & $(5, 200)$ & $0.7$ & $\mathbf{463}$ & $0.8$ & $509$ \\ 
 & $(3,300)$ & $0.7$ & $\mathbf{471}$ & $1.3$ & $498$ \\\hline
\multirow{2}{*}{eil30} & $(4, 150)$ & $3.2$ & $\mathbf{368}$ & $3.0$ & $406$ \\ 
 & $(8,80)$ & $18.5$ & $\mathbf{363}$ & $1.55$ & $479$ \\\hline
 \multirow{2}{*}{att48} & $(5, 8000)$ & $300$ & $34452$ & $13.7$ & $48718$ \\ 
 & $(7, 4000)$ & $300$ & $39748$ & $15.1$ & $56681$ \\\hline
 \multirow{2}{*}{eil51} & $(5, 100)$ & $300$ & $443$ & $20.8$ & $552$ \\ 
 & $(10, 50)$ & $300$ & $491$ & $15.2$ & $609$ \\\hline
 \multirow{2}{*}{eilB76} & $(10, 100)$ & $300$ & $534$ & $10.5$ & $764$ \\ 
 & $(15, 50)$ & $300$ & $557$ & $12.3$ & $774$ \\\hline
 \multirow{2}{*}{eilA101} & $(7, 200)$ & $300$ & $624$ & $8.21$ & $960$ \\ 
 & $(10, 100)$ & $300$ & $649$ & $10.8$ & $945$ \\\hline
 \multirow{2}{*}{eil262} & $(60, 150)$ & $-$ & $-$ & $112$ & $3130$ \\ 
 & $(40, 250)$ & $-$ & $-$ & $115$ & $3050$ \\\hline
\end{tabular}

\label{tab:comparison}
\end{table}

\section{Conclusion}
In this paper, we presented an approximation algorithm for solving the problem of finding a route for a battery-constrained robot in the presence of multiple depots. We also proposed heuristic improvements to the algorithm and tested the algorithm against an ILP on TSPLIB instances. Considering recharging locations with different recharge times or costs can be an interesting direction for future work. Another direction for future work is to consider multiple robots to minimize their total or maximum travel times.
% \section{Multi robot (Future work)}
% %\todo{no results here. I have an algorithm idea, trying to make the approximation ratio work here as well.}

% \subsection{Multi-vehicle multi-recharging depot algorithm}
% \begin{enumerate}
%     \item after getting minimum recharge cover
%     \item make vertices of connnected segments
%     \item split large vertices such that if a vertex has more than $total/r$ segments, split it, 
%     \item the cost between two vertices is the number of recharges required to travel between the two, plus half of the recharges required inside a vertex for both the end points
%     \item run multi vehicle tsp/mst
% \end{enumerate}

\newpage
\bibliographystyle{IEEEtran}
\bibliography{refs}

\begin{thebibliography}{10}
\providecommand{\url}[1]{#1}
\csname url@rmstyle\endcsname
\providecommand{\newblock}{\relax}
\providecommand{\bibinfo}[2]{#2}
\providecommand\BIBentrySTDinterwordspacing{\spaceskip=0pt\relax}
\providecommand\BIBentryALTinterwordstretchfactor{4}
\providecommand\BIBentryALTinterwordspacing{\spaceskip=\fontdimen2\font plus
\BIBentryALTinterwordstretchfactor\fontdimen3\font minus
  \fontdimen4\font\relax}
\providecommand\BIBforeignlanguage[2]{{%
\expandafter\ifx\csname l@#1\endcsname\relax
\typeout{** WARNING: IEEEtran.bst: No hyphenation pattern has been}%
\typeout{** loaded for the language `#1'. Using the pattern for}%
\typeout{** the default language instead.}%
\else
\language=\csname l@#1\endcsname
\fi
#2}}

\bibitem{basilico2012patrolling}
N.~Basilico, N.~Gatti, and F.~Amigoni, ``Patrolling security games: Definition
  and algorithms for solving large instances with single patroller and single
  intruder,'' \emph{Artificial Intelligence}, vol. 184, pp. 78--123, 2012.

\bibitem{asghar2016stochastic}
A.~B. Asghar and S.~L. Smith, ``Stochastic patrolling in adversarial
  settings,'' in \emph{IEEE American Control Conference}, 2016, pp. 6435--6440.

\bibitem{jennings2019study}
D.~Jennings and M.~Figliozzi, ``Study of sidewalk autonomous delivery robots
  and their potential impacts on freight efficiency and travel,''
  \emph{Transportation Research Record}, vol. 2673, no.~6, pp. 317--326, 2019.

\bibitem{wei2018coverage}
M.~Wei and V.~Isler, ``Coverage path planning under the energy constraint,'' in
  \emph{IEEE International Conference on Robotics and Automation (ICRA)}, 2018,
  pp. 368--373.

\bibitem{merino2012unmanned}
L.~Merino, F.~Caballero, J.~R. Mart{\'\i}nez-De-Dios, I.~Maza, and A.~Ollero,
  ``An unmanned aircraft system for automatic forest fire monitoring and
  measurement,'' \emph{Journal of Intelligent and Robotic Systems}, vol.~65,
  no. 1-4, pp. 533--548, 2012.

\bibitem{hayat2017multi}
S.~Hayat, E.~Yanmaz, T.~X. Brown, and C.~Bettstetter, ``Multi-objective uav
  path planning for search and rescue,'' in \emph{IEEE International Conference
  on Robotics and Automation (ICRA)}, 2017, pp. 5569--5574.

\bibitem{khuller2011fill}
S.~Khuller, A.~Malekian, and J.~Mestre, ``To fill or not to fill: The gas
  station problem,'' \emph{ACM Transactions on Algorithms (TALG)}, vol.~7,
  no.~3, pp. 1--16, 2011.

\bibitem{sundar2013algorithms}
K.~Sundar and S.~Rathinam, ``Algorithms for routing an unmanned aerial vehicle
  in the presence of refueling depots,'' \emph{IEEE Transactions on Automation
  Science and Engineering}, vol.~11, no.~1, pp. 287--294, 2013.

\bibitem{nagarajan2012approximation}
V.~Nagarajan and R.~Ravi, ``Approximation algorithms for distance constrained
  vehicle routing problems,'' \emph{Networks}, vol.~59, no.~2, pp. 209--214,
  2012.

\bibitem{friggstad2014approximation}
Z.~Friggstad and C.~Swamy, ``Approximation algorithms for regret-bounded
  vehicle routing and applications to distance-constrained vehicle routing,''
  in \emph{ACM symposium on Theory of computing}, 2014, pp. 744--753.

\bibitem{mitchell2015multi}
D.~Mitchell, M.~Corah, N.~Chakraborty, K.~Sycara, and N.~Michael, ``Multi-robot
  long-term persistent coverage with fuel constrained robots,'' in \emph{IEEE
  International Conference on Robotics and Automation (ICRA)}, 2015, pp.
  1093--1099.

\bibitem{asghar2023multi}
A.~B. Asghar, S.~Sundaram, and S.~L. Smith, ``Multi-robot persistent
  monitoring: Minimizing latency and number of robots with recharging
  constraints,'' \emph{arXiv preprint arXiv:2303.08935}, 2023.

\bibitem{hari2022bounds}
S.~K.~K. Hari, S.~Rathinam, S.~Darbha, S.~G. Manyam, K.~Kalyanam, and
  D.~Casbeer, ``Bounds on optimal revisit times in persistent monitoring
  missions with a distinct and remote service station,'' \emph{IEEE
  Transactions on Robotics}, vol.~39, no.~2, pp. 1070--1086, 2022.

\bibitem{asghar2023risk}
A.~B. Asghar, G.~Shi, N.~Karapetyan, J.~Humann, J.-P. Reddinger,
  J.~Dotterweich, and P.~Tokekar, ``Risk-aware recharging rendezvous for a
  collaborative team of uavs and ugvs,'' in \emph{IEEE International Conference
  on Robotics and Automation (ICRA)}, 2023, pp. 5544--5550.

\bibitem{takahashi2014estimating}
K.~Takahashi, T.~Tsujikawa, K.~Hirose, and K.~Hayashi, ``Estimating the life of
  stationary lithium-ion batteries in use through charge and discharge
  testing,'' in \emph{IEEE International Telecommunications Energy Conference},
  2014, pp. 1--4.

\bibitem{arkin2006approximations}
E.~M. Arkin, R.~Hassin, and A.~Levin, ``Approximations for minimum and min-max
  vehicle routing problems,'' \emph{Journal of Algorithms}, vol.~59, no.~1, pp.
  1--18, 2006.

\bibitem{reinelt1991tsplib}
G.~Reinelt, ``Tsplib—a traveling salesman problem library,'' \emph{ORSA
  Journal on Computing}, vol.~3, no.~4, pp. 376--384, 1991.

\end{thebibliography}
\end{document}